\documentclass{article}

\newcommand{\additionalpackages}{\setlength{\oddsidemargin}{0.25 in}
\setlength{\evensidemargin}{-0.25 in}
\setlength{\topmargin}{-0.6 in}
\setlength{\textwidth}{6 in}
\setlength{\textheight}{8.5 in}
\setlength{\headsep}{0.75 in}
\setlength{\parindent}{0 in}
\setlength{\parskip}{0.1 in}}
\newcommand{\beforeappendix}{\newpage \section*{Acknowledgements}
IB, OS and YM are supported by the European Research Council (ERC) under the European Union’s Horizon 2020 research and innovation program (grant agreement No. 882396), by the Israel Science Foundation and the Yandex Initiative for Machine Learning at Tel Aviv University and by a grant from the Tel Aviv University Center for AI and Data Science (TAD).
OS is also supported by the TAD Excellence Program for Doctoral Students in Artificial Intelligence and Data Science from the Tel Aviv University Center for AI and Data Science (TAD) and from the Israeli Council for Higher Education (CHE) Fellowship for Outstanding PhD Students in Data Science.
}
\newcommand{\afterappendix}{}

\usepackage[numbers]{natbib}
\additionalpackages

\usepackage[utf8]{inputenc} % allow utf-8 input
\usepackage[T1]{fontenc}    % use 8-bit T1 fonts
\usepackage{hyperref}       % hyperlinks, MUST be before cleveref
\usepackage{url}            % simple URL typesetting
\usepackage{booktabs}       % professional-quality tables
\usepackage{amsfonts}       % blackboard math symbols
\usepackage{nicefrac}       % compact symbols for 1/2, etc.Key
\usepackage{microtype}      % microtypography

\usepackage{graphicx} % Required for inserting images
\usepackage{amsmath}
\usepackage{amssymb}
\usepackage{mathtools}
\usepackage{amsthm}
\usepackage{algorithm}
\usepackage{algpseudocode}
\usepackage[dvipsnames]{xcolor}
\usepackage{tikz}
\usepackage{hyperref} % 
\usepackage{cleveref}
\usepackage{dsfont}
\usepackage{enumitem}

\newcommand{\calA}{\mathcal{A}}
\newcommand{\A}{\calA}

\newcommand{\regret}{\operatorname{Regret}}
\newcommand{\calH}{\mathcal{H}}

\newcommand{\C}{\mathcal{C}}
\renewcommand{\citet}[1]{\citeauthor{#1} (\citeyear{#1})}

\newcommand{\curly}[1]{ {\left\{ #1 \right\}}}
\newcommand{\roundy}[1]{ {\left( #1 \right)}}
\newcommand{\squary}[1]{ {\left[ #1 \right]}}
\newcommand{\abs}[1]{ {\left | #1 \right |}}

\newcommand{\E}{\mathbb{E}}
\newcommand{\rr}{\mathbb{R}}
\newcommand{\Expect}[1]{\E\left[#1\right]}

\newcommand{\AlgA}{\texttt{A}}
\newcommand{\AlgB}{\texttt{B}}
\newcommand{\AlgAFull}[3]{\AlgA^{#2}(#1,#3)}
\newcommand{\AlgBFull}[2]{\AlgB(#1,#2)}
\newcommand{\algnameShort}{\texttt{CAOS}}
\newenvironment{proofsketch}{{\bf Proof sketch:}}{\hfill\rule{2mm}{2mm}}

\theoremstyle{plain}
\newtheorem{theorem}{Theorem}[section]

\newtheorem{lemma}[theorem]{Lemma}

\theoremstyle{definition}
\newtheorem{definition}[theorem]{Definition}

\newtheorem{property}[theorem]{Property}
\theoremstyle{remark}
\newtheorem{remark}[theorem]{Remark}

\makeatletter \newcommand{\ForceCrefTypeInEnv}[2]{%
\AddToHook{env/#1/begin}{% 
\let\Cref@oldlabel\label \def\label##1{\Cref@oldlabel[#2]{##1}}% 
}% 
\AddToHook{env/#1/end}{\let\label\Cref@oldlabel}} 
\makeatother
\ForceCrefTypeInEnv{lemma}{lemma} 
\ForceCrefTypeInEnv{definition}{definition} 
\ForceCrefTypeInEnv{proposition}{proposition} 
\ForceCrefTypeInEnv{claim}{claim} 
\ForceCrefTypeInEnv{corollary}{corollary} 
\ForceCrefTypeInEnv{example}{example}
\ForceCrefTypeInEnv{remark}{remark}
\ForceCrefTypeInEnv{property}{property}
\Crefname{lemma}{Lemma}{Lemmas}
\Crefname{definition}{Definition}{Definitions} 
\Crefname{proposition}{Proposition}{Propositions} 
\Crefname{claim}{Claim}{Claims}
\Crefname{corollary}{Corollary}{Corollaries} 
\Crefname{example}{Example}{Examples}
\Crefname{remark}{Remark}{Remarks}
\Crefname{property}{Property}{Properties}

\newcommand{\mail}[1]{\href{mailto:#1}{\color{black} \texttt{#1}}}

\newcommand\asteriskfootnote[1]{%
  \begin{NoHyper}
  \renewcommand\thefootnote{* }\footnotetext{#1}%
  \end{NoHyper}
}

\begin{document}

\title{Collaborating in Multi-Armed Bandits with Strategic Agents}

\author{Idan Barnea$^*$ \\ Tel Aviv University\\ \mail{idan1id@gmail.com}
\and
Ofir Schlisselberg$^*$ \\ Tel Aviv University \\ \mail{ofirs4@mail.tau.ac.il}
\and
Yishay Mansour \\ Tel Aviv University and Google Research \\ \mail{mansour.yishay@gmail.com}}
\date{}

% \author{Idan Barnea$^*$\blackfootnote{Tel Aviv University; \mail{idan1id@gmail.com}}
% \and
% Ofir Schlisselberg$^*$\blackfootnote{Tel Aviv University; \mail{ofirs4@mail.tau.ac.il}} 
% \and
% Yishay Mansour$^\dagger$\blackfootnote{Tel Aviv University and Google Research; \mail{mansour.yishay@gmail.com}}}
\maketitle

\asteriskfootnote{Equal contribution, alphabetical order.}

\begin{abstract}
We study collaborative learning in multi-agent Bayesian bandit problems, where strategic agents collectively solve the same bandit instance. While multiple agents can accelerate learning by sharing information, 
strategic agents might prefer to free-ride and avoid exploration. 
We consider a setting with persistent agents that participate in multiple time periods.
This is in contrast to most previous works on incentives in multi-agent MAB, which assume short-lived agents, namely each agent has a single decision to make and optimizes their expected reward in that single decision. As in the multi-agent MAB model with incentives, our model does not have monetary transfers, and the only incentives are through information sharing.

We propose \texttt{CAOS}, a mechanism that sustains collaboration as a Nash equilibrium while achieving strong regret guarantees. 
Our results demonstrate that collaborative exploration can be sustained purely through information sharing, achieving performance close to that of fully cooperative systems despite strategic behavior.

\end{abstract}

\section{Introduction}

The multi-armed bandit (MAB) framework is a central model for sequential decision-making under uncertainty, capturing the fundamental trade-off between exploration and exploitation \cite{Lattimore_Szepesvari_2020,slivkins2019introduction}. In many modern applications, learning is not performed by a single decision-maker but by multiple agents operating in parallel solving the same MAB instance. Such multi-agent systems can significantly accelerate learning by pooling observations, enabling faster identification of optimal actions. At the same time, their setting introduces a fundamental tension: while exploration benefits the system as a whole, each individual agent prefers to exploit current knowledge rather than incur the cost of exploration.

A prominent line of work studies how incentives play a role in this setting. They consider a principal which coordinates learning across multiple agents through information design \citep{kremer2014implementing,mansour2015bayesian,mansour2016bayesian}.  In these models, a sequence of short-lived myopic agents interacts with a principal, each acting only once. The principal leverages her informational advantage to recommend actions that are incentive-compatible for these myopic agents, thereby effectively guiding exploration.
However, this approach relies critically on the assumption that agents are short-lived and optimize only their immediate reward. In many real-world settings, the same agents interact repeatedly over multiple time steps. In such a setting, the agents can accumulate their own information, and can condition their future behavior on past interactions. In such environments, the principal might no longer hold an informational advantage, and the problem of inducing exploration becomes fundamentally more challenging.

Our persistent-agent setting arises naturally in a variety of applications. 
In navigation systems, users contribute traffic data that improves route recommendations, but have little incentive to explore less certain routes themselves.
In distributed clinical trials \citep{Villar2015-wb,Flores2021-la}, medical institutions repeatedly collect and share data to improve treatment policies, yet each institution may prefer to rely on the experimentation of others.
In the emerging agentic economy \citep{DBLP:journals/sigecom/ImmorlicaLS24-gen-ai-economy,DBLP:journals/cacm/RothschildMHDGIJLSV26-agentic-economy} autonomous LLM agents perform tasks on behalf of users — writing code, conducting research — and could benefit from sharing exploration outcomes with other agents tackling similar problems. However, each agent has an incentive to exploit knowledge generated by others rather than explore.
In all these settings, agents control their own information and interact repeatedly, making strategic behavior unavoidable.

%In the absence of a mechanism to regulate information sharing, such systems can perform poorly. 
Unfortunately, mechanisms that completely share all information might perform poorly.
%The classic work of 
\citet{bolton1999strategic} show that when there is complete information sharing,
%when information is a public good, 
strategic agents tend to under-invest in exploration, leading to inefficient outcomes. \citet{jung2020quantifying} show that a free-rider agent incurs horizon-free regret by relying on the exploration of others. This highlights a fundamental tension—while collaboration enables faster learning, strategic behavior can entirely undermine it.

% In this paper, we study a Bayesian multi-agent bandit setting with strategic agents. Agents interact over multiple time steps, each selecting an arm and observing a reward. They may choose what information to share with others, and cannot be incentivized through monetary transfers or binding contracts. The only tool available to the mechanism designer is control over information flow: what is shared, with whom, and when.
\textbf{Our Contributions.}
In this paper, we study a Bayesian multi-agent bandit setting with strategic agents. The agents interact over multiple time steps. In each time step each agent selects an arm and observes a reward. Agents may choose what information to share with others, and cannot be incentivized through monetary transfers or binding contracts.
We assume that environment-related information — such as chosen arms and realized rewards — is verifiable, while agents retain full discretion over if and how they communicate it.
The only tool available to the mechanism designer is control over information flow: what is shared, with whom, and when.

Our approach is based on using information as an incentive. We propose \algnameShort{} (Collaborating Agents with Optimistic Stopping), a mechanism in which agents dynamically decide whether to remain in the collaborative group.
\algnameShort{} maintains a base algorithm \AlgA{} that it aims to have all agents follow.
At each time step, agents evaluate their expected future reward using an optimistic estimate that assumes collaboration continues. If this value is no better than acting independently, the agent opts out and switches to a single-agent strategy; otherwise, she remains active and shares information.
A carefully designed communication protocol ensures that deviations are detected and that all active agents maintain a common view of the system.
%
%\subsection{Our Contributions}
% Under our interaction protocol, in each time step every agent selects an arm and observes the resulting reward, after which agents communicate with one another.
%
% We show that \algnameShort{} achieves both strategic stability and strong performance guarantees:
%We propose \algnameShort{}, a mechanism for strategic multi-agent bandits that uses information flow as the sole incentive tool. 
Our main contributions are as follows.
\begin{itemize}[leftmargin=*]
    \item \textbf{Problem Formulation:} We introduce a model of multi-agent Bayesian bandits with persistent strategic agents, where collaboration must be sustained without monetary transfers or binding contracts. The only mechanism available is the control of information sharing, leading to a repeated game in which agents decide dynamically whether to participate in collaborative exploration.
    \item \textbf{Multi-agents Mechanism} We propose \algnameShort{}, a mechanism for strategic multi-agent bandits that uses information flow as the sole incentive tool. We show the following properties for \algnameShort{}:
    \begin{itemize}[leftmargin=*]
       \item \textbf{Equilibria:} Following \algnameShort{} constitutes a \emph{Nash equilibrium}; no agent can benefit by unilaterally deviating from the prescribed strategy. Moreover, we show that a simple extension of \algnameShort{} forms a \emph{subgame perfect equilibrium}.
    % \item \textbf{Regret Guarantee:} For any symmetric algorithm \AlgA{}, \emph{the regret incurred under \algnameShort{} is no greater than the regret achieved when all agents follow \AlgA{} without deviation.}
    % For example, UCB achieves $\Expect{O(K\log T/(m\Delta))}$ and Thompson Sampling achieves $O(\sqrt{KT/m})$, where $T$ is the time horizon, $K$ arms, $m$ agents, and $\Delta$ is the gap.
    % We further demonstrate that low regret is attainable for asymmetric algorithms, such as asymmetric Successive Elimination (regret of $\Expect{O(K\log T/(m\Delta))}$).
    \item \textbf{Regret Guarantee:} For any symmetric algorithm \AlgA{}, \emph{the regret under \algnameShort{} is no greater than the regret when all agents follow \AlgA{} without any deviation}. Concretely, UCB achieves $\mathbb{E}[O(K\log T/(m\Delta))]$ and Thompson Sampling achieves $O(\sqrt{KT/m})$, where $T$ is the time horizon, $K$ the number of arms, $m$ the number of agents, and $\Delta$ the suboptimality gap. We further show that low regret is attainable beyond symmetric algorithms — for instance, asymmetric Successive Elimination achieves $\mathbb{E}[O(K\log T/(m\Delta))]$.
    \end{itemize}
\end{itemize}

    % \item \textbf{Symmetric Algorithms:} When the underlying multi-agent algorithm $\AlgA$ is symmetric (i.e., prescribes the same action to all agents), \algnameShort{} preserves its performance guarantees. In particular, it achieves the same regret as $\AlgA$ in the non-strategic setting (e.g., for UCB).
    % \item \textbf{Asymmetric Algorithms:} For general (possibly asymmetric) algorithms, we show that each agent achieves at least the performance of the worst-performing agent under $\AlgA$ at each time step. We further demonstrate that this guarantee remains meaningful for standard algorithms such as Successive Elimination, yielding near-optimal regret bounds.

Together, these guarantees show that \algnameShort{} simultaneously ensures strategic compatibility and low regret, nearly matching the performance of fully collaborative systems.

\subsection{Related Work}

\citet{kremer2014implementing}, \citet{mansour2015bayesian} study a setting in which a principal seeks to minimize regret in a Bayesian bandit problem. At each step, the principal interacts with a new agent and recommends an action. This setup introduced a tension similar to ours: while each agent prefers to exploit the current knowledge to maximize their own payoff, the principal is also interested in exploration to improve future decisions.
To address this, they introduce the notion of a Bayesian incentive-compatible (BIC) strategy, ensuring that agents are always incentivized to follow the principal’s recommendations. \citet{mansour2016bayesian} extend this framework to allow for $m$ agents at each step; however, as in the original model, a new set of $m$ agents arrives in every time step.
\citet{DBLP:journals/ior/SimchowitzS24-Incentives-rl} study this setting where each agent arrives exactly once and extend it to reinforcement learning (MDPs).
\citet{banihashem2026banditsociallearningexploration} extend the model by dividing the horizon into episodes, with a new agent arriving at the start of each episode. When episodes are of length one, this recovers the setting of prior work.
These works differ from ours in that we assume the same set of agents participates throughout the entire time horizon. This introduces additional challenges: agents interact repeatedly and can influence one another's behavior, creating strategic interdependencies absent in the short-visit setting.
% Moreover, since agents are present for the full horizon, they have their own incentive to explore---yet their exploratory actions also affect the experiences of other agents, further complicating the mechanism design.

In game theory, cooperation refers to settings where agents form coalitions and enter binding agreements, which differs from our non-cooperative setting in which each agent acts independently \citep{a-course-in-game-theory}. 
\citet{krishnamurthy2026creatorincentivesrecommendersystems} studied cooperation in repeated games with multi-agent MABs, using binding agreements.
We note that while ``cooperation'' carries a specific meaning in game theory, our setting nonetheless incentivizes collaborative behavior.

The principal--agents setting with monetary transfers has been studied previously \citep{Ramakrishnan-k2024collaborative}. In our setting, by contrast, incentivization relies solely on controlling the flow of information.

A large body of work studies non-strategic multi-agent bandit problems, in which agents fully collaborate and share information to accelerate learning \citep{Cesa-Bianchi-cooperation,ChakrabortyCDJ17,baron-coop-mab-graph,Martinez-RubioK19-ucb, dubey2020cooperative, LandgrenSL21-ucb, DBLP:conf/icml/Lancewicki0M22, wang2022achieving-ucb,yang2024cooperative, zhang2025nearoptimal, barnea2026provablecooperativemultiagentexploration} (see \cite{zhang2021multi} for a survey). These works demonstrate substantial gains from parallel exploration. Our work complements this line by studying how to recover such benefits in the presence of strategic behavior.

\section{Problem Formulation}

\begin{algorithm}[t]
\floatname{algorithm}{Protocol}
\caption{Interaction Protocol}
\label{alg:interaction-protocol}
\begin{algorithmic}[1]
\Require Number of arms $K$, horizon $T$, prior $Q$, unknown instance $\mathcal{I} \sim Q$
\For{$t = 1,2,\dots,T$}
    \State Each agent $i$ selects an arm $a_i^t \in [K]$, and observes reward $r_i^t \sim \mathcal{I}(a_i^t)$
    \State The agents communicate with each other
\EndFor
\end{algorithmic}
\end{algorithm}

\paragraph{Bayesian Multi-Armed Bandit.}
We consider a multi-agent Bayesian multi-armed bandit (MAB) setting in which the environment is drawn from a prior distribution. Let $Q$ be a prior over a set of problem instances, where each instance $\mathcal{I}$ specifies a reward distribution for every arm. At the beginning of the interaction, an instance $\mathcal{I} \sim Q$ is sampled once and remains fixed throughout the horizon.
At each time step $t = 1,\dots,T$, each agent $i \in [m]$ selects an arm $a_i^t \in [K]$ and receives an i.i.d.~reward $r_i^t(a_i^t) \sim \mathcal{I}(a_i^t)$.
The Bayesian regret of agent $i$ is defined as
\[
\regret_i = \mathbb{E}_{\mathcal{I} \sim Q} \left[
\mathbb{E} \left[
\sum_{t=1}^T \left( \max_{a \in [K]} \mathbb{E}[r_i^t(a)] - \mathbb{E}[r_i^t(a_i^t)] \right)
\,\middle|\, \mathcal{I}
\right]
\right].
\]

% The expected reward of agent $i$ from time $t=1$ is denoted
% \[
% V_i = \Expect{\sum_{t=1}^T r^t_i(a^t_i)}.
% \]

\paragraph{Interaction Protocol.}
At each time step, every agent decides which arm to choose, and what to communicate to other agents (see \Cref{alg:interaction-protocol}).
At each time step, agents may exchange messages.
Our focus is on incentivization through the information flow.
We assume that environment-related information is verifiable; in particular, chosen actions and realized rewards are verifiable (e.g., the environment produces signatures), but agents may withhold any kind of information.
We do not assume truthfulness beyond this: strategic agents may choose what to share and may lie about the information they hold. This distinction is central to the subgame perfect equilibrium analysis.

% A strategy $\pi_i$ for agent $i$ maps the current time $t$ and posterior $q_t^i$ to: an arm $a_i^t$, a set of agents with whom to share the chosen arm, after observing rewards and others' actions, a set of agents with whom to share her reward.
A \textbf{strategy} $\pi_i$ for agent $i$ maps the current time $t$ and history $\calH^t_i$ to: an arm $a_i^t$, and communication policy.
We denote a strategy profile by $\pi = (\pi_i)_{i \in [m]}$. The cumulative reward from time step $t$ onward of agent $i$ is denoted \[
V_i^t(\pi) = \sum_{s=t}^T r^{s}_i(\pi(s,\calH_i^{s})_i).
\]
We use $V_i(\pi)$ to indicate $V_i^1(\pi)$.

We denote by $\pi_{-i}$ the strategy profile of all agents except $i$.
We denote by $V_i(\pi'_i, \pi_{-i})$ the reward of agent $i$ when she plays $\pi'_i$ and the rest play $\pi_{-i}$.

\paragraph{Equilibria.}
A strategy profile $\pi = (\pi_i)_{i \in [m]}$ is a Nash equilibrium if for every agent $i \in [m]$ and every alternative strategy $\pi'_i$,
\[
\mathbb{E}_{\mathcal{I} \sim Q} \left[ V_i(\pi) \right]
\;\ge\;
\mathbb{E}_{\mathcal{I} \sim Q} \left[ V_i(\pi'_i, \pi_{-i}) \right].
\]
It is a \emph{subgame perfect equilibrium} if the same condition holds after every possible history.

\paragraph{Algorithms.}
We assume to have two known algorithms:
(i) a multi-agent algorithm $\AlgA^\C$ for every $\C \subseteq [m]$, and
(ii) a single-agent algorithm $\AlgB$.

We denote by $\AlgAFull{\calH}{\C}{t} \in [K]^{\abs{\C}}$ the vector of arms selected by $\AlgA$ at time $t$ under history $\calH$ if the set of agents is $\C$, and by $\AlgBFull{\calH}{t} \in [K]$ the arm selected by $\AlgB$. We denote $\pi_\AlgB$ the strategy of all agents playing $B$ and not sharing, and $\pi_{\AlgA}$ to be the strategy of playing $\AlgA$ and always share. 
We assume $\AlgB$ is the strategy that each agent plays if she is in a single agent environment, i.e the strategy profile $\pi_\AlgB$ is a Nash-equilibrium.

\section{The CAOS Mechanism}
\label{sec:caos}

Our goal is not merely to establish the existence of an equilibrium, but to design an equilibrium that achieves strong regret guarantees. Indeed, equilibria always exist in our setting. For example, if all agents play the single-agent optimal algorithm and do not share any information, this constitutes a Nash equilibrium.
Another example is when a single agent runs the optimal single-agent algorithm and broadcasts its observations to all others, while every other agent free-rides — sharing nothing and contributing no exploration of their own. This too constitutes a Nash equilibrium: no agent can unilaterally improve her rewards by deviating. 
%Yet the individual regret remains high.

Such equilibria fail to exploit the potential benefits of collaboration and may result in suboptimal collective performance. To address this, we aim to construct an equilibrium that achieves the performance of a target multi-agent algorithm $\AlgA$, which assumes that agents are non-strategic and always collaborate. Our challenge is to sustain such collaboration in the presence of strategic behavior.

Our mechanism, Collaborating Agents with Optimistic Stopping (\algnameShort{}), achieves this by combining information sharing with a dynamic participation rule. Informally, \algnameShort{} maintains at every time step a set of \emph{active} agents that continue to collaborate. Active agents follow the multi-agent algorithm $\AlgA$ and share information with one another. Agents that stop being active switch permanently to the single-agent algorithm $\AlgB$ and do not share their actions or rewards.

Our policy decides whether an agent remains active by using a recursive procedure that we call \emph{Optimistic Expected Reward} (\texttt{OER}). At a high level, \texttt{OER} computes for every agent the expected continuation reward she would obtain if all agents continue according to \algnameShort{} from the current history onward. The \algnameShort{} mechanism then compares this continuation value to the continuation value of stopping and switching to $\AlgB$.

\subsection{The CAOS protocol}

We begin with a short operational description of \algnameShort{}. We denote by $\pi$ the strategy profile in which all agents follow \algnameShort{} (\Cref{alg:caos-main}). The formal algorithm is given at \Cref{alg:caos}.

\begin{algorithm}
\caption{Collaborating Agents with Optimistic Stopping (\algnameShort{}) -- (informal, see \Cref{alg:caos})}
\label{alg:caos-main}
\begin{algorithmic}[1]
\State \textbf{Input:} time step $t$, history $\calH^t$, agent $i$.
\State $\rho \gets OER(t,\calH^t)$ \Comment{Optimistic continuation rewards }

\If{$\rho_i = \text{value of switching to }\AlgB$}
    \State Play according to $\AlgB$; do not share any information.
    \State \Return
\EndIf
\State $\C_t \gets \curly{i\colon\; \rho_i > \text{value of switching to $\AlgB$}}$
\State Play the recommended action of $\AlgA$ $a_i^t$ and observe reward $r_i^t$.

\State Share $a_i^t$ with agents in $\C_t$ and observe their actions.

\State $\mathcal{W}_t \gets \{j \in \C_t \mid j \text{ shared her action and followed } \AlgA\}$

\State Share reward $r_i^t$ with $\mathcal{W}_t$ and observe their rewards.

\State Share all newly received rewards with $\mathcal{W}_t$.

\end{algorithmic}
\end{algorithm}

For the presentation of \algnameShort{} in this section, we assume access to such an \texttt{OER} oracle. Later in the section we define \texttt{OER} formally and show that, under compliant play, it indeed computes the correct continuation value.

The algorithm works as follows. At the beginning of time step $t$, each agent evaluates her continuation value using \texttt{OER}. If this value equals the continuation value of switching to $\AlgB$, then the agent stops collaboration and from now on behaves as a single agent. 
% (This behaves as a stopping time \citep{peskir2006optimal}.)
Otherwise, she remains active for that time step, follows the recommendation of $\AlgA$, and participates in the communication protocol.

The communication protocol has three parts. First, active agents share the arm they played. This allows every active agent to verify that the others indeed followed the action prescribed by $\AlgA$. Second, only agents that passed this verification exchange their observed rewards. Third, every active agent forwards all newly received rewards to the other active agents. The purpose of this final step is to guarantee that all active agents end the time step with the same information, and therefore with the same posterior and the same continuation problem in the next time step. (The justification for each of these communication parts is discussed in \Cref{sec:nash}).

\subsection{The OER procedure}

We now define the Optimistic Expected Reward (\texttt{OER}) procedure formally. Recall that the set of active agents is determined from the history.

\begin{definition}\label{def:C_t}
Given a history up to time $t$, $\C_{t-1}$ is defined to be the set of agents that, up to time $t-1$, played according to $\AlgA$ and shared their reward as required by \algnameShort{}.
\end{definition}

% Thus, by definition, once an agent leaves $\C_{t-1}$ she can never re-enter it.

The \texttt{OER} procedure recursively evaluates the continuation reward of every agent under the optimistic assumption that, among the currently active agents, the largest subset that all prefer continuing indeed keeps collaborating.

\begin{algorithm}
\caption{Optimistic Expected Reward (\texttt{OER}) -- (Informal, see \Cref{alg:OER})}
\label{alg:OER-main}
\begin{algorithmic}[1]
\Require time $t$, history $\calH^t$.

\If{$t = T+1$}
    \State \Return $0$ for all agents
\EndIf

\State Let $\C_{t-1}$ be the set of currently active agents (from the history).

\For{candidate sets $\C \subseteq \C_{t-1}$ in decreasing size}
    \State Compute $\rho$ as:
    \begin{itemize}
        \item For $i \in \C$: immediate expected reward under $\AlgA$ if $\C$ plays it, plus continuation value given by $OER(t+1,\cdot)$
        \item For $i \notin \C$: value of switching to $\AlgB$
    \end{itemize}

    \If{all $i \in \C$ prefer continuing over switching to $\AlgB$}
        \State \Return $\rho$
    \EndIf
\EndFor
\end{algorithmic}
\end{algorithm}

A few remarks are in order. First, the procedure is recursive: for agents in the candidate set $\C$, it evaluates the immediate expected reward under $\AlgA$ plus the continuation value returned by \texttt{OER} at the next history. Agents outside $\C$ are assumed to stop immediately and obtain the continuation value of $\AlgB$.

Second, \texttt{OER} is \emph{optimistic} in the following sense. It scans subsets of $\C_{t-1}$ from largest to smallest and returns the first subset for which every member strictly prefers continuation over switching to $\AlgB$. Hence, among all self-sustaining collaborating sets, it selects the largest one. This maximizes the amount of collaboration consistent with individual incentives.

Finally, note that \texttt{OER} always returns a continuation value at least as large as $\E[V_i^t(\pi_\AlgB)\mid \calH^t]$ for every agent $i$, since the candidate set $\C=\emptyset$ is always feasible. Therefore, in \Cref{alg:caos-main}, the event
\[
\rho_i = \E[V_i^t(\pi_\AlgB)\mid \calH^t]
\]
exactly captures the case in which agent $i$ should stop.

\begin{remark}
    The \texttt{OER} algorithm as stated is fully optimistic: it attempts to retain the largest group of collaborating agents, at the cost of computational efficiency. An efficient variant, in which only the empty and full groups are checked (i.e., any single agent that stops causes all others to stop too, see \Cref{alg:EOER}), retains the optimism assumption and achieves the same regret guarantees.
\end{remark}

We next show that if all agents follow \algnameShort{}, then \texttt{OER} indeed computes the correct continuation reward.

\begin{lemma}\label{lem:OER_asymmetric}
Assume all agents follow $\pi$ (\Cref{alg:caos-main}). For every agent $i$, time $t$, and history $\calH^t$,
\[
    \E[V_i^t(\pi)\mid \calH^t] = OER(t,\calH^t)_i.
\]
\end{lemma}

\begin{proof}[Proof sketch]
Since all agents followed $\pi$ up to time step $t-1$, every agent in $\C_{t-1}$ has received exactly the same information: by the last communication step in \algnameShort{}, all rewards observed by active agents are forwarded to all currently active agents.
Therefore, all agents in $\C_{t-1}$ have the same history $\calH^t$, and hence compute the same set $\C_{t-1}$ and the same output of $OER(t,\calH^t)$.

In particular, they all compute the same set $\C_t$ selected by \texttt{OER}. Because each active agent follows the recommendation of \texttt{OER} when deciding whether to remain active, the set $\C_t$ chosen by the procedure is exactly the set of agents that truly continue with $\AlgA$ at time $t$. Agents outside $\C_t$ switch to $\AlgB$.

Thus, the actual play from time $t$ onward is exactly the play assumed by the \texttt{OER} recursion: agents in $\C_t$ obtain their expected immediate reward under $\AlgA$, while agents outside $\C_t$ obtain the continuation value of $\AlgB$. Since this is true for every $t'\ge t$, the recursion holds and the returned reward is indeed exactly $\E[V_i^t(\pi)\mid \calH^t]$.
\end{proof}

The lemma formalizes the key fixed-point property of the mechanism: under compliant play, the optimistic recursion is not merely hypothetical, but exactly describes the true continuation rewards.

\subsection{CAOS is a Nash equilibrium}\label{sec:nash}

We now show that the strategy profile induced by \algnameShort{} is stable.

\begin{theorem}\label{thm:caos_nash}
The strategy profile $\pi$ in which all agents follow \algnameShort{} is a Nash equilibrium.
\end{theorem}

\begin{proofsketch}
Assume toward contradiction that there exists an agent $j$ and an alternative policy $\pi'_j$ such that agent $j$ strictly improves her expected total reward by deviating from $\pi$ to $\pi'_j$ while all other agents follow $\pi$. Fix some history that they differ on, and denote $t$ to be the first time at which $\pi'_j$ differs from $\pi$.
We distinguish between two cases.

\textbf{Case 1: Agent $j\notin C_t$.}
By definition of \algnameShort{}, she no longer shares information and no active agent shares information with her. Moreover, because all agents in $\C_{t}$ have the same history by the beginning of time step $t$, they all know that $j$ is inactive and therefore exclude her from future collaboration.

Hence, from time step $t$ onward, agent $j$ effectively acts as a single agent. Among all such single-agent continuations, $\AlgB$ is optimal by assumption. Therefore deviating from $\pi$ cannot improve her reward in this case.

\textbf{Case 2: Agent $j\in C_t$.}
Since $j$ is active, under $\pi$ she follows $\AlgA$ and the communication protocol. There are two ways in which $\pi'_j$ can first differ from $\pi$: either in the action she plays, or in the information she shares.

\emph{Deviation in the played arm.}
Suppose first that agent $j$ plays an arm different from the one prescribed by $\AlgA$. The first communication step of \algnameShort{} requires agents to share their played arm before rewards are shared. Hence the other active agents detect immediately that $j$ did not follow $\AlgA$, and therefore exclude her from the set $\mathcal W_t$ with whom rewards are shared.

This ordering is crucial. If rewards were shared before actions were verified, then an agent could exploit in the current time step, still receive the others' exploration rewards, and only afterwards leave the collaboration. Such a deviation would invalidate the optimistic recursion underlying \texttt{OER}. \algnameShort{} rules this out by revealing actions first. 

Once deviation in the played arm is detected, agent $j$ can no longer benefit from collaborative information sharing in time step $t$, and from that point on she is effectively reduced to a single-agent continuation. Among all such continuations, $\AlgB$ is optimal. Therefore the best deviation of this kind is no better than stopping and switching to $\AlgB$. But by \Cref{lem:OER_asymmetric}, \algnameShort{} already compares the continuation reward of remaining active to the continuation reward of $\AlgB$, and chooses the better of the two. Thus agent $j$ cannot gain by deviating in the played arm. 

\emph{Deviation in information sharing.}
It remains to consider the possibility that agent $j$ plays the prescribed arm but deviates in the sharing phase. A priori, such a deviation may look beneficial: by selectively sharing information, agent $j$ might try to create asymmetric information among the other active agents, and this could potentially alter their future stopping decisions.

The last communication step of \algnameShort{} is designed precisely to prevent this. Every active agent forwards all newly received rewards to all currently verified active agents. Consequently, among agents that continue to be treated as active, the information held at the end of the time step is identical. In particular, all active agents have the same history and therefore the same posterior and the same \texttt{OER} computation in the next time step.

If agent $j$ has already played according to $\AlgA$ in the current time step, then deviating in the sharing phase cannot help her. If she withholds information, she is no longer considered active by the entire active group, and effectively transitions to acting alone from that point onward, receiving no further information from the other agents. Thus, by not sharing, she strictly reduces the information available to her future decisions. In contrast, sharing does not reduce her future option value: she still retains the ability to stop in subsequent time steps, while benefiting from additional information provided by the other active agents. Therefore, after following $\AlgA$ in the current time step, it is weakly optimal for her to share all required information and postpone the stopping decision to the next time step, where she may still choose to stop. Thus, deviating in the sharing phase cannot strictly improve her payoff either.

We conclude that no unilateral deviation can strictly improve the payoff of any agent. Therefore $\pi$ is a Nash equilibrium.
\end{proofsketch}

\begin{remark}
The equilibrium proof highlights the role of each part of the communication protocol. Sharing played arms before rewards eliminates profitable ``free-ride-now'' deviations. Sharing rewards only among verified active agents ensures that only compliant agents enjoy collaborative exploration. Finally, forwarding all newly received rewards guarantees that the active agents maintain a common history, which is essential both for the correctness of \texttt{OER} and for the equilibrium argument.
\end{remark}

\subsection{Subgame Perfect Equilibrium}

The analysis above establishes that $\pi$ is a Nash equilibrium. However, this guarantee applies only to histories that arise when all agents follow $\pi$. In general, off-equilibrium histories may lead to divergence in agents' information, which can break the equilibrium structure.

When agents possess heterogeneous information, they may hold different posteriors and therefore face different continuation incentives. In this case, the notion of the active group is no longer a consensus — each agent maintains her own conception of it. While it would be desirable to reach consensus over both the active group and the posteriors, agents with an informational advantage may be reluctant to reveal it. 
For example, an inactive agent may share information with one active agent but not with another. The latter agent has no way to detect it, since verifiability applies only to environment-generated signals, there is no mechanism to verify that a message was \emph{not} sent. As a result, some agents may remain unaware that others possess superior information, giving those who do hold it little incentive to disclose it.

To address this issue, we extend the mechanism to ensure subgame perfection. The key idea is to detect any deviation from $\pi$ and collapse the system to a fallback equilibrium. Specifically, we augment the mechanism with an initial communication phase in which each agent reports whether she is aware of any past deviation from $\pi$. The full mechanism is given in \Cref{alg:sp-caos}.

The crucial observation (proved in \Cref{lem:psi}) is that any deviation from $\pi$ is observed by at least two agents. Therefore, no agent can conceal such a deviation: if one agent reports it, the other has no incentive to deny it, since the deviation will be revealed regardless. As a result, truthful reporting forms an equilibrium of this communication phase.

If a deviation is detected, all agents switch to playing $\AlgB$ and cease information sharing, which constitutes a Nash equilibrium. Otherwise, the history is consistent with $\pi$, and the original analysis applies. In this case, the equilibrium is subgame perfect.

\section{Regret Guarantees}
\label{sec:regret}

We now analyze the performance of \algnameShort{}. While \Cref{sec:caos} shows that agents have no incentive to deviate, it does not by itself guarantee that the resulting collaboration achieves low regret. The regret of \algnameShort{} depends on the structure of the underlying multi-agent algorithm $\AlgA$.

We first show that when $\AlgA$ is symmetric, \algnameShort{} preserves its performance guarantees. We then show that for general asymmetric algorithms, such a guarantee is impossible, and provide a weaker bound. Finally, we illustrate these results with concrete examples.

\subsection{Symmetric algorithms}

We say that $\AlgA$ is \emph{symmetric} if it prescribes the same action to all active agents at every time step.

\begin{theorem}\label{thm:regret_symmetric}
Assume $\AlgA$ is symmetric. Then for every agent $i$,
\begin{align*}
    \regret_i(\pi) \le \regret_i(\pi_\AlgA).
\end{align*}
\end{theorem}

\begin{proofsketch}
Under symmetry, all active agents receive the same recommendation from $\AlgA$ and therefore face identical continuation problems. In particular, at every history they compute the same \texttt{OER} value and make the same decision whether to continue or stop. Thus, either all active agents continue together or all stop together.

From the perspective of any agent, the optimistic evaluation performed by \texttt{OER} can therefore assume that if she continues, all other agents also continue. Since \texttt{OER} evaluates continuation optimistically, it effectively assumes that collaboration persists until the end of the horizon. As a result, an agent can only benefit from stopping, and \algnameShort{} already compares this option to continuing and selects the better one. This implies that the continuation value under \algnameShort{} is at least as large as under $\AlgA$.
By \Cref{lem:max-reward-min-regret}, maximizing reward is equivalent to minimizing regret, and hence this is also a regret guarantee.
\end{proofsketch}

We leverage this theorem to recover known regret bounds when $\AlgA$ is instantiated with standard symmetric algorithms. In particular, we demonstrate this for UCB, which achieves (\Cref{thm:ucb})
\begin{align*}
    O\!\left(\mathbb{E}\!\left[\sum_{a\in[K]}\frac{\log T}{m\Delta_a}\right] + K\right),
\end{align*}
and Thompson Sampling, which achieves (\Cref{thm:TS})
\begin{align*}
    \tilde O\!\left(\sqrt{KT/m} + K\right).
\end{align*}

% We note that these are the bounds achieved when running these algorithms with non-strategic agents.

\subsection{Asymmetric algorithms}

When $\AlgA$ is asymmetric, different agents may be assigned different roles, and in particular some agents may bear a larger exploration cost than others.
In this case, the strong guarantee of \Cref{thm:regret_symmetric} no longer holds.
Rather than guaranteeing low individual regret, we will only guarantee a related regret notion, which we call $\regret_{\max}$:

\[
 \regret_{\max}(\pi) = \mathbb{E}_{\mathcal{I} \sim Q} \squary{ \Expect{\sum_{t=1}^T (\max_{a\in [K]}\Expect{r^t_i(a)} - \min_i \E\squary{r_i^{t}(\pi(\calH^{t}, t))\mid \calH^{t}}) \mid \mathcal{I}}}.
\]

Notice that $\regret_i(\pi) \le \regret_{\max}(\pi)$ for all $i$, with equality when \AlgA{} is symmetric.  
%Note that in $\regret_{\max}(\pi)$ diffferent $t$ have different minimizing $i$.

We state that \algnameShort{} does guarantee this new regret notion.
\begin{theorem}\label{thm:regret_asymmetric}
For every agent $i$,
% \begin{align*}
%     \E_Q[V_i(\pi)] \ge \E_Q[V_{\min}(\pi_\AlgA)].
% \end{align*}
%
\begin{align*}
    \regret_i(\pi) \le \regret_{\max}(\pi_\AlgA).
\end{align*}

\end{theorem}
The proof of this theorem is deferred to the appendix (see \Cref{thm:regret-asym-appndx}).

\begin{remark}
One can wonder whether individual regret is the right benchmark for asymmetric algorithms.
The next example demonstrates that asymmetric algorithms can be strategically problematic to follow. Even if they achieve low expected individual regret they might simultaneously exhibit high maximum regret, making them ones that no agent would willingly follow.
For example, consider an algorithm $\AlgA$ that with $n \le m$ agents guarantees regret $4R/n$ per agent, where $\AlgB$ achieves regret $R$. Suppose $\AlgA$ selects one agent uniformly at random as a \emph{victim}, who incurs regret $2R$, while each other agent incurs $2R/(n-1)$ (so the expected regret per agent is $4R/n$). Then the victim, once notified, strictly prefers to stop and switch to $\AlgB$. This repeats recursively, so agents drop out one by one at time step $t=1$, and eventually all switch to $\AlgB$, each incurring regret $R$.
This example shows that for asymmetric algorithms, \algnameShort{} cannot in general guarantee each agent the same reward as under $\AlgA$. Despite this limitation, \algnameShort{} still provides a meaningful guarantee.
%particularly when $\AlgA$ is ``fair'' in the sense that all agents receive comparable performance.
\end{remark}

\subsection{Examples}

We now illustrate the implications of \Cref{thm:regret_asymmetric} with two examples. Proofs are in Appendix~\ref{apx:examples}.

\textbf{Fixed-arm explore--exploit.}
Consider the algorithm that assigns each agent a fixed arm during the exploration phase; once the algorithm decides to exploit, all agents switch to the empirically best arm (see \Cref{alg:mafaee}).
\begin{theorem}
The Bayesian regret of $\algnameShort$ when $\AlgA$ is \Cref{alg:mafaee} is bounded by
\[
\tilde{O}\left(T^{2/3} K^{1/3} m^{-1/3}\right).
\]
\end{theorem}
This result matches the results of the natural multi-agent extension of the classical explore--then--exploit analysis (see, e.g., \cite{slivkins2019introduction}).

This example highlights the robustness of \algnameShort{}: even when the underlying algorithm is somewhat crude and unfair, the mechanism is able to sustain collaboration and recover the expected non-strategic multi-agent gains.
When $\AlgA$ is ``fair'' in the sense that all agents receive comparable performance, stronger guarantees achieved.

\textbf{Successive Elimination.}
We next consider a standard extension of the Successive Elimination algorithm, where the algorithm assigns different arms from the set of active arms to each agent (see \Cref{alg:SE}).

\begin{theorem}
The Bayesian regret of $\algnameShort$ when $\AlgA$ is \Cref{alg:SE} is bounded by
\[
O\!\left(\E_{\Delta_{\min} \sim Q}\left[\frac{K \log T}{m \Delta_{\min}}\right]\right).
\]
\end{theorem}

The regret decreases as $1/m$, can be lower than $K$, and nearly matches the classical bound for multi-agent Successive Elimination.

\section{Discussion and Future Work}

This work studies collaboration in self-interested multi-agent multi-armed bandit settings, where information control serves as the sole mechanism for incentivizing others. Our results demonstrate that collaborative exploration among strategic agents can be sustained through information alone, while maintaining strong regret guarantees.
Several directions remain open. First, our regret bounds in the asymmetric case may not be tight; obtaining sharper bounds, or proving impossibility results for mechanisms that rely solely on information control, would be a valuable contribution. Second, our model assumes that environment-generated signals are verifiable. Relaxing this assumption presents significant challenges, as agents could misreport their observations. Finally, extending our framework to richer online learning settings, such as contextual bandits, linear bandits, and Markov decision processes is a natural and important direction.
We hope this work offers a step toward designing learning systems for strategic agents, a problem that grows increasingly relevant in the era of agentic AI.

\beforeappendix

\newpage
\bibliographystyle{abbrvnat}
\bibliography{coop_mab_game}
\newpage
\appendix

\section{The CAOS Mechanism and Equilibria}

\subsection{Optimistic Expected Reward (OER)}
In this section we present the \texttt{OER} procedure, and prove the main lemma that is related to it \Cref{lem:oer_asymmetric_appndx}.

\begin{definition}[Restatement of \Cref{def:C_t}]
Given a history up to time $t$, $\C_{t-1}$ is defined to be the set of agents that played according to $\AlgA$ and shared their reward.
\end{definition}

\begin{definition}
We use here $\AlgAFull{\calH}{\C}{t}$ to indicate the actions of $\AlgA$ given the history $\calH$ and time $t$ for the set of active agents $\C$. 
\end{definition}

\begin{algorithm}
\caption{Optimistic Expected Reward (\texttt{OER})}
\label{alg:OER}
\begin{algorithmic}[1]
\Require time $t$, history $\calH^t$.
\If{$t = T+1$}
    \State\Return $(0, 0, \dots, 0) \in \rr^m$
\EndIf
\State Find $\C_{t-1}$ according to $\calH^t$ by \Cref{def:C_t}.
\For{$\C \subseteq \C_{t-1}$ ordered by $\abs{\C}$ (largest first)}
    \State $\boldsymbol{a} \gets \AlgAFull{\calH^t}{\C}{t}$
    \State $\rho_i \gets \begin{cases}
        \E_{\boldsymbol{r}\sim \calH^t(\boldsymbol{a})}\squary{\boldsymbol{r}^i + OER(t+1,\, \calH^t \cup \curly{\boldsymbol{a}\colon \boldsymbol{r}})_i} & i\in \C\\
        \E[V_i^t(\pi_B) \mid \calH^t] & i\notin \C
    \end{cases}$
    \If{$\forall i\in \C\quad \rho_i > \E[V_i^t(\pi_\AlgB) \mid \calH^t]$} \Comment{Always true for $\C = \emptyset$}
           \State \Return $\rho$ 
    \EndIf
\EndFor
\end{algorithmic}
\end{algorithm}

\begin{algorithm}
\caption{Efficient Optimistic Expected Reward (\texttt{EOER})}
\label{alg:EOER}
\begin{algorithmic}[1]
\Require time $t$, history $\calH^t$.
\If{$t = T+1$}
    \State\Return $(0, 0, \dots, 0) \in \rr^m$
\EndIf
\State $\boldsymbol{a} \gets \AlgAFull{\calH^t}{[m]}{t}$
\State $\rho_i \gets \E_{\boldsymbol{r}\sim \calH^t(\boldsymbol{a})}\squary{\boldsymbol{r}^i + OER(t+1,\, \calH^t \cup \curly{\boldsymbol{a}\colon \boldsymbol{r}})_i}$
\State $V_\AlgB \gets \E[V_i^t(\pi_\AlgB) \mid \calH^t]$
\If{$\forall i\in [m]\quad \rho_i > V_\AlgB$}  
           \State \Return $\rho$ 
    \Else
        \State \Return $(V_\AlgB, V_\AlgB, \dots, V_\AlgB)$
    \EndIf
\end{algorithmic}
\end{algorithm}

\begin{remark}
We analyze both versions of \texttt{OER} together, while only assuming that \Cref{pro:OER_moreB,pro:OER_good,lem:OER-apx,pro:OER_allin} is true. 
\end{remark}

\begin{property}\label{pro:OER_moreB}
For every $\calH^t,t,i$:
\begin{align*}
    OER(\calH^t, t)_i \ge \E[V_i^t(\pi_\AlgB) \mid \calH^t]
\end{align*}
\end{property}

\begin{property}\label{pro:OER_good}
For every $\calH^t,t$, if there is $i\in[m]$ such that $OER(\calH^t, t)_i = \E[V_i^t(\pi_\AlgB) \mid \calH^t]$, then there is an agent $j$ (possibly $i=j$) such that:
\begin{align*}
    \E_{\boldsymbol{r}\sim \calH^t(\boldsymbol{a})}\squary{\boldsymbol{r}^j + OER(t+1,\, \calH^t \cup \curly{\boldsymbol{a}\colon \boldsymbol{r}})_j} \le \E[V_j^t(\pi_\AlgB) \mid \calH^t]
\end{align*}
\end{property}

\begin{property}\label{pro:OER_allin}
For every $\calH^t,t$, if $\C_t = [m]$, then for every agent $i$:
\begin{align*}
    OER(t, \calH^t)_i = \E_{\boldsymbol{r}\sim \calH^t(\boldsymbol{a})}\squary{\boldsymbol{r}^i + OER(t+1,\, \calH^t \cup \curly{\boldsymbol{a}\colon \boldsymbol{r}})_i}
\end{align*}
\end{property}

\begin{lemma}[Restatement of \Cref{lem:OER_asymmetric}]\label{lem:OER-apx}
\label{lem:oer_asymmetric_appndx}
Assume all agents follow $\pi$. For every agent $i$, step $t$ and history $\calH^t$:
\begin{align*}
    \E[V_i^t(\pi)\mid \calH^t] = OER(t, \calH^t)_i
\end{align*}
\end{lemma}
\begin{proof}
We'll prove by backwards induction on $t$. In time $T+1$ the game has ended so there is indeed no rewards. Now assume for $t+1$ and we'll prove for $t$.

For every $i\notin \C_t$, agent $i$ will play $\AlgB$ in step $t$. From \Cref{def:C_t}, we have also $i\notin \C_{t'}$ for every $t'\ge t$, which means that
\begin{align*}
    \E[V_i^t(\pi)\mid \calH^t] = \E[V_i^t(\pi_\AlgB)\mid \calH^t] = OER(t, \calH^t)_i
\end{align*}
The last is by the definition of $\C_t$ and \Cref{pro:OER_moreB}.

We now note that in both versions of $OER$, every $i\in \C_t$ has the same value of:
\begin{align*}
    OER(t, \calH^t) = \E[\boldsymbol{r}^i + OER(t+1, \calH^t \cup \curly{\boldsymbol{a}\colon\boldsymbol{r}}) \mid \boldsymbol{r} \sim \calH^t(\boldsymbol{a)}]
\end{align*}
With $\boldsymbol{a}$ being $\AlgAFull{\calH^t}{\C_t}{t}$.

Thus:
\begin{align*}
    \E[V_i^t(\pi)\mid \calH^t] &= \E[\boldsymbol{r}^i + \E[V_i^{t+1}(\pi)\mid \calH^t \cup \curly{\boldsymbol{a}\colon\boldsymbol{r}}] \mid \boldsymbol{r} \sim \calH^t(\boldsymbol{a)}]\\
    &= \E[\boldsymbol{r}^i + OER(t+1, \calH^t \cup \curly{\boldsymbol{a}\colon\boldsymbol{r}}) \mid \boldsymbol{r} \sim \calH^t(\boldsymbol{a)}]\\
    &= OER(t, \calH^t)
\end{align*}
The first is the tower rule and the second is from the induction assumption.
\end{proof}

\subsection{Nash Equilibrium}
In this section we present the Collaborating Agents with Optimistic Stopping (\algnameShort{}) procedure and prove Nash Equilibrium.

\begin{algorithm}
\caption{Collaborating Agents with Optimistic Stopping (\algnameShort{})}
\label{alg:caos}
\begin{algorithmic}[1]
\State \textbf{Input:} time step $t$;
History $\calH^{t}$;
Agent id $i$;
\If{There is agent $j$ that didn't follow $\algnameShort$ in $\calH^t$}
    \State Play according to $\AlgB$; Do not share the chosen arm and observed reward.
    \State \Return
\EndIf

\State $\rho \gets OER(\calH^t)$
\If{$\rho_i = \E[V_i^t(\pi_\AlgB) \mid \calH^t]$}
    \State Play according to $\AlgB$; Do not share the chosen arm and observed reward.
    \State \Return
\EndIf
\State $\C_t \gets \curly{i\colon\; \rho_i > \E[V_i^t(\pi_\AlgB) \mid \calH^t]}$
\State Get recommended action $a_i^t = \AlgAFull{\calH^t}{\C_t}{t}_i$, play it and observe reward $r^t_i$.
\State Send $a_i^t$ to $\C_t$ and observe played arms.
\State $\mathcal{W}_t \gets \curly{j \in \C_t \mid j\text{ shared played arm and } a_j^t = \AlgAFull{\calH^t}{\C_t}{t}_j}$
\State Share reward with $\mathcal{W}_t$ and observe shared rewards.
\State Share all newly arrived rewards with $\mathcal{W}_t$.
\end{algorithmic}
\end{algorithm}

\begin{theorem}[Restatement of \Cref{thm:caos_nash}]
The strategy profile $\pi$ in which all agents follow \algnameShort{} is a Nash equilibrium.
\end{theorem}
\begin{proof}
Assume by contradiction that it is not Nash Equilibrium. 
Fix time $t$ and history $\calH^t$ with $\pi_i' \ne \pi_i$ such that:
\begin{align}\label{eq:nash_contradiction}
    \E[V_i(\pi) \mid \calH^t] < \E[V_i(\pi'_i, \pi_{-i}) \mid \calH^t]
\end{align}
We assume w.l.o.g that $\pi_i'$ is the optimal policy for $i$ starting from $\calH^t$.

Notice that if $i\in \C_{t-1}$, agent $i$ saw the same history as all other agents in $\C_{t-1}$, and thus they share the same $\C_{t}$.

Assume first that the difference is in the arm decision. If $i\notin \C_{t}$, agent $i$ will get no feedback but her own (since all agents in $\C_t$ has the same $\C_t$ and knows $i\notin \C_t$). Thus, $\pi=\pi_\AlgB$, which is assumed to be a Nash equilibrium.

If $i\in \C_{t}$, $\pi_i$ chooses $\AlgAFull{\calH^t}{\C_{t}}{t}$. If $\pi'_i$ chooses otherwise, it won't get more info. We have:
\begin{align*}
    \E[V_t^i(\pi) \mid \calH^t] &= OER(t,\calH^{t})_i\\
    &\ge V_i(\pi_\AlgB)\\
    &\ge V_i(\pi_i', \pi_{-i})
\end{align*}
the second inequality is from \Cref{pro:OER_moreB} and the last line is since playing $\AlgB$ alone is Nash equilibrium.

Hence, the difference is not the arm decision. I.e., $\pi'$ chooses the same arms as $\pi$.
Assume the difference is in the sharing.

Now assume that $i\notin \C_t$ is true and $\pi'$ shares (arm or reward). Since it doesn't affect the sharing of the other agents, it won't change anything and you won't get the strong inequality in \eqref{eq:nash_contradiction}.

Now assume that $i\in \C_t$ and $\pi'$ doesn't share arm. At that point, the agent knows $r_i^t$ but won't know the reward of the other agents. We have:
\begin{align*}
    \E[V_i^{t+1}(\pi'_i, \pi_{-i}) \mid \calH^t \cup \curly{a_i^t\colon r_i^t}] &= \E[V_i^{t+1}(\pi_\AlgB, \pi_{-i}) \mid \calH^t \cup \curly{a_i^t\colon r_i^t}]\\
    &\le \E[\E[V_i^{t+1}(\pi_\AlgB, \pi_{-i}) \mid \calH^t \cup \curly{\boldsymbol{a}\colon\boldsymbol{r}}] \mid \boldsymbol{r}^{-i}]\\
    &= \E[\E[V_i^{t+1}(\pi_\AlgB) \mid \calH^t\cup \curly{\boldsymbol{a}\colon\boldsymbol{r}}] \mid \boldsymbol{r}^{-i}]\\
    &\le \E[OER(t+1, \calH^t\cup \curly{\boldsymbol{a}\colon\boldsymbol{r}})) \mid \boldsymbol{r}^{-i}]\\
    &= \E[\E[V_i^{t+1}(\pi) \mid \calH^t\cup \curly{\boldsymbol{a}\colon\boldsymbol{r}}] \mid \boldsymbol{r}^{-i}]\\
    &= \E[V_i^{t+1}(\pi) \mid \calH^t\cup \curly{\boldsymbol{a}\colon\boldsymbol{r}}]
\end{align*}
The first is from the optimality of $\pi_\AlgB$, we can assume that $\pi'_i$ is the same as $\pi_\AlgB$ from the next step ($t+1$). The second is since an agent can only gains from more information. The third is since there's no sharing in $\pi$ after $i$ didn't share, the policy of others doesn't matter. The fourth is by definition of $OER$. The fifth is \Cref{lem:OER_asymmetric}. The sixth is tower rule.

Now assume that $i\in\C_t$ and $\pi'$ doesn't share reward. At that point, the agent knows $r_i^t$ and will continue with the knowledge of the full $\boldsymbol{r}$. Thus, to compare them we need to have an expectation for every option of $\boldsymbol{r}^{-i}$ (the rewards of the other agents). We have, in the same way as above:
\begin{align*}
    \E[\E[V_i^{t+1}(\pi'_i, \pi_{-i}) \mid\calH^t\cup \curly{\boldsymbol{a}\colon\boldsymbol{r}}] \ \mid \boldsymbol{r}^{-i}] &= \E[\E[V_i^{t+1}(\pi_\AlgB) \mid \calH^t\cup \curly{\boldsymbol{a}\colon\boldsymbol{r}}] \mid \boldsymbol{r}^{-i}]\\
    &\le \E[OER(t+1, \calH^t\cup \curly{\boldsymbol{a}\colon\boldsymbol{r}}) \mid \boldsymbol{r}^{-i}]\\
    &= \E[\E[V_i^{t+1}(\pi) \mid \calH^t\cup \curly{\boldsymbol{a}\colon\boldsymbol{r}}] \mid \boldsymbol{r}^{-i}]
\end{align*}
\end{proof}

\subsection{Subgame Perfect Equilibrium}
In this section we present the extension for \algnameShort{} procedure to achieve Subgame Perfect Equilibrium, and prove the equilibrium.

\begin{definition}
Let $\psi_i(\calH)$ be the event that there is an agent (including $i$ itself) that deviated from $\algnameShort$ in $\calH$ from the knowledge of the $i$th agent, except if the deviation was that an agent $j$ with $j\notin \C_t$ did not play by $\AlgB$ in time step $t$.
\end{definition}

\begin{algorithm}
\caption{Subgame Perfect Collaborating Agents with Optimistic Stopping (\texttt{SP-CAOS})}
\label{alg:sp-caos}
\begin{algorithmic}[1]
\State \textbf{Input:} time step $t$;
History $\calH^{t}$;
Agent id $i$;
\If{$\psi_i(\calH^t)$}
    \State Send $True$
\Else
    \State Send $False$
\EndIf
\State Receive all indications from agents
\If{One of the agents send $True$}
    \State Play according to $\AlgB$; Do not share the chosen arm and observed reward.
\Else
    \State Continue with $\algnameShort$
\EndIf

\If{There is agent $j$ that didn't follow $\algnameShort$ in $\calH^t$}
    \State Play according to $\AlgB$; Do not share the chosen arm and observed reward.
\EndIf
\end{algorithmic}
\end{algorithm}

\begin{lemma}\label{lem:psi}
For any history $\calH$, if there is $i$ such that $\psi_i(\calH)$ is true, there is $j\ne i$ such that $\psi_j(\calH)$ is also true.
\end{lemma}
\begin{proof}
The agent that deviated from $\pi$ in the history has this information in her history by definition. Thus, we can assume w.l.o.g that $i$ is the first agent that deviated, and we want to show that there exists an agent $j\ne i$ that also has this in her history. Fix $t$ to be the time step in the history that agent $i$ deviated.

The important point is that since we consider the first deviation, by the definition of $\algnameShort$, all agents in $\C_{t}$ has the same information. Thus, all agents in $\C_{t}$ knows what are the exact actions that $i$ should do in this step.

First assume that $i\in \C_t$ and the deviation was in arm choice. If agent $i$ didn't share which arm she played, then this is itself a deviation and all the other agents know about it. If she did share, all the other agents know that she played the wrong arm, so they know about the deviation. 

Second, assume the deviation was with sharing. There are 4 possible deviations - (1) $i\in \C_t$ and didn't share with some $j\in \C_t$, (2) $i\notin \C_t$ and shared with some $j\in \C_t$, (3) $i \in \C_t$ and shared with $j\notin \C_t$ and (4) $i\notin \C_t$ and shared with some $j\notin \C_t$.

In case (1), since all agents in $\C_t$ share the same knowledge, agent $j$ knows that $i$ should have shared with her. In case (2) since $j\in \C_t$ she knows that $i\notin \C_t$, and thus knows that $i$ shouldn't have shared with her. In cases (3) + (4), agent $j$ knows that no one should share with her since it is inactive. 
\end{proof}

\begin{theorem}
The strategy profile $\pi$ in which all agents follow \Cref{alg:sp-caos} is a subgame perfect equilibrium.
\end{theorem}
\begin{proof}
Assume by contradiction that it is not subgame perfect, i.e there exists an agent $i$ , strategy $\pi'_i$ and history $\calH^t$ such that:
\begin{align*}
    \E[V_i(\pi) \mid \calH^t] < \E[V_i(\pi'_i, \pi_{-i}) \mid \calH^t]
\end{align*}

First assume $\psi_i(\calH^t)$ is $False$ but $\pi'_i$ sends $True$. From now on no other agent will share information with agent $i$, and thus its expected reward is bounded by $\E[\pi_\AlgB \mid \calH^t]$. From \Cref{lem:OER_asymmetric,pro:OER_moreB}, we have:
\begin{align*}
    \E[V_i(\pi) \mid \calH^t] &= OER(\calH^t, t)\\
    &\ge E[V_i(\pi_\AlgB) \mid \calH^t] \\
    &\ge \E[V_i(\pi'_i, \pi_{-i}) \mid \calH^t]
\end{align*}
Which is a contradiction.

Now assume $\psi_i(\calH^t)$ is $True$ but $\pi'_i$ sends $False$. From \Cref{lem:psi} there is an agent $j$ that will send $False$, and thus it doesn't matter what $i$ sends.

Now assume the deviation is in $\algnameShort$. If we reached there it means that all agents followed $\algnameShort$ by now (except inactive agents playing $\AlgB$). One can see that the Nash Equilibrium proof of $\algnameShort$ does not depend on the fact that in the history inactive agents played $\AlgB$, and thus this concludes the proof.
\end{proof}

\section{Regret Bounds}
In this section we derive the regret bounds.

Denote the minimum reward as
\begin{align*}
    V_{\min}^t(\pi) = \sum_{s=t}^T \min_{i}\E\squary{r_i^{s}(\pi(\calH^{s}, s)) \mid \calH^{s}}.
\end{align*}

\begin{lemma}\label{lem:regret_asymmetric}
For every agent $i$, step $t$ and history $\calH^t$ such that $\C_{t-1} = [m]$:
\begin{align*}
    OER(t, \calH^t)_i \ge \E[V_{\min}^t(\pi_\AlgA) \mid \calH^t]
\end{align*}
\end{lemma}
\begin{proof}
We'll prove by backward induction. At time $T+1$ both return $0$ for all agents which is the base of the induction. Assume true for $t+1$. Denote $\boldsymbol{a} \gets \AlgAFull{\calH^t}{[m]}{t}$, and denote agent $j$ to be the one with the lowest reward in $\boldsymbol{a}$ (in expectation over $\calH$).

If $\C_t = [m]$, for every $i$:
\begin{align}\label{eq:all_stay}
\begin{aligned}
    OER(t, \calH^t)_i &= \E_{\boldsymbol{r}\sim \calH^t(\boldsymbol{a})}\squary{\boldsymbol{r}^i + OER(t+1,\, \calH^t + (\boldsymbol{a}\colon \boldsymbol{r}))}\\
    &\ge \E_{\boldsymbol{r}\sim \calH^t(\boldsymbol{a})}\squary{\boldsymbol{r}^i + \E[V_{\min}^{t+1}(\pi_\AlgA) \mid \calH^t + (\boldsymbol{a}\colon \boldsymbol{r})] } \\
    &\ge \E_{\boldsymbol{r}\sim \calH^t(\boldsymbol{a})}\squary{\boldsymbol{r}^j + \E[V_{\min}^{t+1}(\pi_\AlgA) \mid \calH^t + (\boldsymbol{a}\colon \boldsymbol{r})] }\\
    &= \E[V_{\min}^t(\pi_A) \mid \calH^t]
\end{aligned}
\end{align}
The first is from \Cref{pro:OER_allin}, the second is from the induction assumption (it is allowed since $\C_t = [m]$), the third is by the definition of agent $j$ and the last is tower rule.

If $\C_t \ne [m]$, from \Cref{pro:OER_good} there exists agent $j$ such that:
\begin{align}\label{eq:B_good}
\begin{aligned}
    \E[V_{j}^t(\pi_\AlgB) \mid \calH^t] &\ge \E_{\boldsymbol{r}\sim \calH^t(\boldsymbol{a})}\squary{\boldsymbol{r}^j + OER(t+1,\, \calH^t + (\boldsymbol{a}\colon \boldsymbol{r}))_j}\\
    &\ge \E[V_{\min}^t(\pi_A) \mid \calH^t]
\end{aligned}
\end{align}
The last is in the same way as in \Cref{eq:all_stay}.

Now notice that by the definition of $OER$ we have:
\begin{align}\label{eq:OER_good}
    OER(t, \calH^t)_i \ge \E[V_{i}^t(\pi_\AlgB) \mid \calH^t]
\end{align}

Since $\AlgB$ is the same for all the agents, \Cref{eq:B_good,eq:OER_good} concludes the proof.
\end{proof}

\begin{theorem}[Restatement of \Cref{thm:regret_asymmetric}]
\label{thm:regret-asym-appndx}
For every agent $i$:
\begin{align*}
    % \E_Q[V_i(\pi)] \ge \E_Q[V_{\min}(\pi_\AlgA)]
    \regret_i(\pi) \le \regret_{\max}(\pi_\AlgA).
\end{align*}
\end{theorem}
\begin{proof}
Directly from applying \Cref{lem:regret_asymmetric,lem:OER_asymmetric} with $t=1$, and from \Cref{lem:max-reward-min-regret}.
\end{proof}

\begin{theorem}[Restatement of \Cref{thm:regret_symmetric}]
Assume $\AlgA$ is symmetric, then:
\begin{align*}
    \E_Q[V_i(\pi)] \ge \E_Q[V_i(\pi_\AlgA)]
\end{align*}
\end{theorem}
\begin{proof}
Directly from \Cref{thm:regret_asymmetric} and the fact that $V_{\min}(\pi_\AlgA) = V_i(\pi_\AlgA) $ for every $i$ if $\AlgA$ is symmetric.
\end{proof}

\section{Examples}\label{apx:examples}
In this section we provide several multi-agent algorithms and analyze their regret under the \algnameShort{} mechanism.

\subsection{Successive Elimination}
In this section we present an extension to the Successive Elimination (SE) algorithm \citet{Even-DarMM06} for multiple agents. This extension acts exactly as SE, but at each time step it distributes the arms among the agents, and hence can provide a regret guarantee even lower than the number of arms $K$. We assume for simplicity that $m/n$ is an integer for every $n\le K$.

\begin{algorithm}[H]
\caption{Multi-Agent Successive Elimination (MASE)}
\label{alg:SE}
\begin{algorithmic}[1]
\Require $m \ge K$, horizon $T$
\State Initialize active set $\mathcal{A} \gets [K]$
\State For all $a$: $n_a \gets 0$, $\widehat{\mu}_a \gets 0$

\For{$t=1,\dots,T$}
    \For{each $a \in \mathcal{A}$}
        \State Assign $\frac{m}{|\mathcal{A}|}$ agents to arm $a$
        \State Observe rewards and update $n_a$, $\widehat{\mu}_a$
    \EndFor

    \State Define $\mathrm{rad}_a \gets \sqrt{\frac{2\log(TK)}{n_a}}$

    \State Eliminate arms:
    \[
    \mathcal{A} \gets 
    \left\{
    a \in \mathcal{A} :
    \widehat{\mu}_a + \mathrm{rad}_a
    \ge 
    \max_{b \in \mathcal{A}} (\widehat{\mu}_b - \mathrm{rad}_b)
    \right\}
    \]
\EndFor

\end{algorithmic}
\end{algorithm}

\begin{theorem}
The Bayesian regret of $\algnameShort$ when $\AlgA$ is \Cref{alg:SE} is bounded by \[O\roundy{\E_{\Delta_{min}\sim Q}\squary{\frac{K\log(T)}{m\Delta_{min}}}}.\]
\end{theorem}

\begin{proof}
From \Cref{thm:regret_asymmetric} it is enough to bound the regret w.r.t $V_{\min}$. Let $\regret_{max}$ be that regret. 

From [Equation 1.7]\cite{slivkins2019introduction}, for each arm $a$ after we have $n_a$ samples, w.p $1-\frac{1}{T}$:
\begin{align}
\label{eq:slivkins}
\begin{aligned}
    \Delta_a &\le 4\sqrt{\frac{2\log(T)}{n_a}} \\
    \implies n_a &\le \frac{32\log(T)}{\Delta_a^2}
\end{aligned}
\end{align}

Since there are a maximum of $K$ active arms, each active arm gets at least $m/K$ samples per each step, which means that at time $t$ it has at least $tm/K$ samples. Thus, arm $a$ must be eliminated by:
\begin{align}
       \frac{tm}{K} &\le  \frac{32\log(T)}{\Delta_a^2} \notag\\
       \implies t &\le \frac{32K\log(T)}{m\Delta_a^2} \label{eq:inactive_bound}\\ 
       \implies \Delta_a &\le \sqrt{\frac{32K\log(T)}{mt}} \label{eq:delta_bound}
\end{align}

\eqref{eq:inactive_bound} means that by $\tau = \frac{32K\log(T)}{m\Delta_{min}^2}$ all suboptimal arms are inactive. Thus, the regret of the max agent is bounded by:
\begin{align*}
    \E[\regret_{max}] &\le \sum_{t=1}^{T}\max_{a \in \A_t} \Delta_a \\
    &= \sum_{t=1}^{\tau}\max_{a\in\A_t} \Delta_a\\
    &\le \sum_{t=1}^{\tau} \sqrt{\frac{32K\log(T)}{mt}} \tag{Using \eqref{eq:delta_bound}}\\
    &\le 2\sqrt{\frac{32K\log(T)}{m}\frac{32K\log(T)}{m\Delta_{min}^2}}\\
    &= O\roundy{\frac{K\log(T)}{m\Delta_{min}}}
\end{align*}

All of the above is conditioned on \eqref{eq:slivkins} being true. Since it is true w.p $1-\frac{1}{T}$, it doesn't asymptotically change the expected regret, which concludes the proof.
\end{proof}

\subsection{Fixed Arm Explore Exploit}

In this section, we present an algorithm that is inherently sub-optimal, even within a non-strategic framework; yet, agents continue to follow it. This algorithm serves as an extension of the Explore-Then-Exploit paradigm.

During the exploration phase, the algorithm assigns each agent a fixed arm. Once this phase concludes and the exploitation phase begins, all agents transition to the arm that yielded the highest empirical mean. Despite the "unfair" nature of this coordination, we demonstrate that the standard regret bounds for non-strategic multi-agent systems are still successfully achieved when integrated with \algnameShort{}.

\begin{algorithm}[H]
\caption{Multi Agent Fixed Arm Explore Exploit (MAFAEE)}
\label{alg:mafaee}
\begin{algorithmic}[1]
\Require Number of agents $m$, number of arms $K$, horizon $T$, exploration length $N$
\State Assume $\frac{m}{K}$ is an integer
\State Partition the agents into $K$ groups, every $\frac{m}{K}$ playing an arm

\For{$t=1,\dots,N$}
    \State Agent $i$ plays assigned arm
\EndFor

\For{$a=1,\dots,K$}
    \State Compute the empirical mean $\widehat{\mu}_a$ of arm $a$ using all observations collected by agents 
\EndFor

\State Let
\[
\widehat{a}
\in
\arg\max_{a \in [K]} \widehat{\mu}_a
\]
be the empirically best arm

\For{$t=N+1,\dots,T$}
    \State  All agents play arm $\widehat{a}$
\EndFor

\end{algorithmic}
\end{algorithm}

\begin{theorem}
The Bayesian regret of $\algnameShort$ when $\AlgA$ is \Cref{alg:mafaee} with $N = \frac{T^{2/3}(K\log(T))^{1/3}}{m^{1/3}}$ is bounded by $\le O\roundy{\frac{T^{2/3}\roundy{K\log(T)}^{1/3}}{m^{1/3}}}$.
\end{theorem}
\begin{proof}
From \Cref{thm:regret_asymmetric} it is enough to bound the regret w.r.t $V_{\min}$. Let $\regret_{max}$ be that regret. 

Fix an instance and let $\mu$ be the (unknown) vector of expected reward of the arms.

From \cite{slivkins2019introduction}[Equation 1.2], w.p $1-\frac{2}{T}$, after we have $\frac{Nm}{K}$ samples for each arm:
\begin{align*}
    \abs{\mu_a - \widehat{\mu}_a} \le \sqrt{\frac{2K\log(T)}{mN}}
\end{align*}

Thus:
\begin{align*}
    \mu_{\widehat{a}} &\ge \widehat{\mu}_{\widehat{a}} - \sqrt{\frac{2K\log(T)}{mN}}\\
    &\ge \widehat{\mu}_{a^*} - \sqrt{\frac{2K\log(T)}{mN}} \\
    &\ge \mu_{a^*} -  2\sqrt{\frac{2K\log(T)}{mN}}\\
    \implies \Delta_{\hat{a}} &\le 2\sqrt{\frac{2K\log(T)}{mN}}
\end{align*}

Thus:
\begin{align*}
    \regret_{max} \le N + (T - N)2\sqrt{\frac{2K\log(T)}{mN}} \le N + 2T\sqrt{\frac{2K\log(T)}{mN}}
\end{align*}

Setting $N = \frac{T^{2/3}(K\log(T))^{1/3}}{m^{1/3}}$ we get:
\begin{align*}
    \regret_{max} = O\roundy{\frac{T^{2/3}\roundy{K\log(T)}^{1/3}}{m^{1/3}}}
\end{align*}

The above is true w.p $1 - \frac{2}{T}$, which means that the counter won't affect the regret. Since it is true for every mean vector $\mu$ it is also true in Bayesian regret.
\end{proof}

\subsection{UCB}

In this section, we present a direct multi-agent extension of the UCB algorithm \citet{AuerCFS02}.
At each time step, the arm with the highest Upper Confidence Bound (UCB) is simultaneously assigned to every agent.

\begin{algorithm}[H]
\caption{Multi-Agent UCB (MAUCB)}
\label{alg:UCB}
\begin{algorithmic}[1]
\Require $m \ge K$, horizon $T$
\State For all $a$: $n_a \gets 0$, $\widehat{\mu}_a \gets 0$

\For{$t=1,\dots,T$}
    \State Define $\mathrm{rad}_a \gets \sqrt{\frac{2\log(TK)}{n_a}}$
    \State Assign action to all agents:
    \[
        a_t \gets \arg\max_a \curly{\widehat{\mu}_a + \mathrm{rad}_a}
    \]
    \State Observe rewards and update $n_a$, $\widehat{\mu}_a$
\EndFor

\end{algorithmic}
\end{algorithm}

\begin{theorem}\label{thm:ucb}
The Bayesian regret of $\algnameShort$ when $\AlgA$ is \Cref{alg:UCB} is bounded by \[O\roundy{\E_{\Delta}\squary{\sum_a \frac{8\log(T)}{m\Delta_{a}}} + K}.\]
\end{theorem}
\begin{proof}
Since \Cref{alg:UCB} is symmetric, from \Cref{thm:regret_symmetric} we need to bound its regret for the non-strategic case.

For every time step $t$,
from \cite{slivkins2019introduction}[Equation 1.4] we have w.p $1-\frac{1}{T}$:
\begin{align*}
    \Delta_{a_t} &\le 2\sqrt{\frac{2\log(T)}{n_{a_t}^t}}\\
    n_{a_t}^t &\le \frac{8\log(T)}{\Delta_{a_t}^2}
\end{align*}

Fix arm $a$ and let $t_a$ be the last time it was sampled. We have:
\begin{align*}
    n_a^T = n_a^{t_a} \le \left\lceil\frac{8\log(T)}{\Delta_{a}^2} \right \rceil \le \frac{8\log(T)}{\Delta_{a}^2}+m
\end{align*}

Which means that each agent sampled it a maximum of $\frac{8\log(T)}{\Delta_{a}^2m}$ times. Thus, for every agent $i$, w.p $1-\frac{1}{T}$:
\begin{align*}
    \E[\regret_i] &= \E_{\Delta}\squary{\sum_a \roundy{\frac{8\log(T)}{m\Delta_{a}^2} + 1}\Delta_a} \\
    &= \E_{\Delta}\squary{\sum_a \frac{8\log(T)}{m\Delta_{a}} + \Delta_a} \\
    &\le \E_{\Delta}\squary{\sum_a \frac{8\log(T)}{m\Delta_{a}}} + K
\end{align*}

The expected regret in the bad event that happens w.p $1/T$ is bounded by $1$ which doesn't affect the regret asymptotically.
\end{proof}

\subsection{Thompson Sampling}

In this section, we present a direct multi-agent extension of Thompson Sampling \citet{thompson1933likelihood}, where agents sample from the posterior a shared instance and select the arm with the highest expected reward.
For this section, we assume the agents have access to shared randomness.

\begin{algorithm}[H]
\caption{Multi-Agent Thompson Sampling (MATS)}
\label{alg:TS}
\begin{algorithmic}[1]
\Require Number of actions $K$, number of agents $m$, Horizon $T$, prior $Q$

\For{$t=1,\dots,T$}
    \State Sample a problem instance $\mathcal{I}$ from the posterior:
    \[
        \mathcal{I} \sim Q(\cdot \mid \calH^t)
    \]
    \State Assign action to all agents:
    \[
        a_t \gets \arg\max_a \curly{\E[\mathcal{I}_a]}
    \]
    \State Observe rewards and update the posterior
\EndFor

\end{algorithmic}
\end{algorithm}

\begin{theorem}\label{thm:TS}
The Bayesian regret of $\algnameShort$ when $\AlgA$ is \Cref{alg:TS} is bounded by $\tilde{O}\roundy{\sqrt{\frac{KT}{m}} + K}$.
\end{theorem}
\begin{proof}
Since \Cref{alg:TS} is symmetric from \Cref{thm:regret_symmetric} we need to bound its regular regret.

From \cite{slivkins2019introduction}[Lemma 3.10]:
\begin{align*}
    \E[\regret_i] = O\roundy{\sqrt{\log(T)}}\E\squary{\sum_{t=1}^T \min\curly{\sqrt{\frac{1}{n_{a_t}^t}}, 1}}
\end{align*}

We have:
\begin{align*}
    \sum_{t=1}^T \min\curly{\sqrt{\frac{1}{n_{a_t}^t}}, 1} &= \sum_a 1 + \sum_{t\colon a_t=a, n_a^t>0}^T \sqrt{\frac{1}{n_{a}^t}}\\
    &= K + \sum_a\sum_{j=1}^{\abs{t\colon a_t=a}} \sqrt{\frac{1}{mj}}\\
    &= K + \sum_aO\roundy{\frac{\sqrt{n_a}}{m}}\\
    &\le K + O\roundy{\frac{\sqrt{K\sum _a n_a}}{m}}\\
    &= K + O\roundy{\frac{\sqrt{KTm}}{m}}\\
    &= K + O\roundy{\sqrt{\frac{KT}{m}}}
\end{align*}
\end{proof}

% \section{Naive algorithms}

% \newcommand{\bestGreedy}{\texttt{Best-Single-Greedy}}

% \bestGreedy{} may not be Nash:
% Let us consider a naive strategy, that demonstrates why always sharing may not be a good idea.
% This is a classic free-riding example.
% We call \bestGreedy{} the algorithm that everyone plays the arm if they were to play alone from now on.
% In the simple form of strategy they always share.
% This is not necessary Nash: in one arm bandit, they may sample the unknown mean arm in the first time step, even though the fixed arm has higher mean, resulting in free riding.

% The regret from \algnameShort{} is strictly better than \bestGreedy{}, even if it is Nash:
% When we allow to stop sharing when one deviate, it is indeed contained in our \algnameShort{} algorithm, setting \texttt{A} to \bestGreedy{}, and therefore a Nash.
% But there are multi-agent algorithms that are strictly better than \bestGreedy{}.
% For example, one arm bandit when a single agent never samples the unknown mean arm.
% An algorithm that simply treat each time step with $m$ agents and is designed for one arm bandit performs better.

\section{Auxiliary Lemmas}

\begin{lemma}
\label{lem:max-reward-min-regret}
    Maximizing the expected reward is equivalent to minimizing regret.
\end{lemma}

\begin{proof}
    The expected reward of the choices of agent $i$ is
    \[
    \Expect{\sum_{t=1}^T r^t_i(a^t_i)},
    \]
    where $a^t_i$ is the action agent $i$ chooses at time $t$, and the expectation is over all the randomness, including the prior.

    Recall that the regret for a given instance $\mathcal{I}$, is
    \[
    \regret_i(\mathcal{I}) = \Expect{\sum_{t=1}^T \max_{a\in[K]}\Expect{r^t_i(a)} - \sum_{t=1}^T r^t_i(a^t_i) \mid \mathcal{I}}.
    \]
    Since the rewards are i.i.d. among agents and time steps, we can denote $\max_{a\in [K]}\Expect{r_i^t(a)} = \max_{a\in [K]}\Expect{r_j^s(a)} := \mu^\star$. 
    I.e., $\mu^\star$ is the maximum expected mean reward among the arms.

    The Bayesian regret is
    \begin{align*}
    \regret_i = \Expect{\regret_i(\mathcal{I})} &= \Expect{T \mu^\star - \Expect{\sum_{t=1}^T r^t_i(a^t_i) \mid \mathcal{I}}}
    \\&= \Expect{\Expect{T\mu^\star \mid \mathcal{I}} - \Expect{\sum_{t=1}^T r^t_i(a^t_i)\mid \mathcal{I}}}
    \\&= \Expect{\Expect{\Expect{T\mu^\star \mid \mathcal{I}}} - \Expect{\sum_{t=1}^T r^t_i(a^t_i)\mid \mathcal{I}}}
    \\&= T\Expect{\mu^\star} - \Expect{\sum_{t=1}^T r^t_i(a^t_i)}.
    \end{align*}
    The first term, $T\Expect{\mu^\star}$, is independent of the agent choices, hence maximizing $\Expect{\sum_{t=1}^T r^t_i(a^t_i)}$ minimizes the regret.
\end{proof}

\newpage
\afterappendix

\end{document}